\documentclass{article}

\usepackage{times}
\usepackage{graphicx} % more modern
\usepackage{subfigure}

\usepackage{natbib}

\usepackage{algorithm}
\usepackage{algorithmic}

\usepackage{hyperref}

\usepackage[accepted]{icml2016}

\usepackage{comment}
\usepackage{amsthm}
\usepackage{mathtools}
\usepackage[colorinlistoftodos,prependcaption]{todonotes}

\newtheorem{theorem}{Theorem}
\newtheorem{lemma}{Lemma}
\newtheorem{proposition}{Proposition}
\newtheorem{corollary}{Corollary}
\newtheorem{definition}{Definition}

\newtheorem{remark}{Remark}
\newtheorem{assumption}{Assumption}

\icmltitlerunning{Optimal Experiment Design for Causal Discovery from Fixed Number of Experiments}

\begin{document}

\twocolumn[
\icmltitle{Optimal Experiment Design for Causal Discovery from\\Fixed Number of Experiments}

\icmlauthor{AmirEmad Ghassami}{ghassam2@illinois.edu}
\icmladdress{Department of ECE, Coordinated Science Laboratory,\\
			University of Illinois at Urbana-Champaign Urbana, IL 61801 USA}
\icmlauthor{Saber Salehkaleybar}{sabersk@illinois.edu}
\icmladdress{Coordinated Science Laboratory,\\
			University of Illinois at Urbana-Champaign Urbana, IL 61801 USA}
\icmlauthor{Negar Kiyavash}{kiyavash@illinois.edu}
\icmladdress{Department of ECE and ISE, Coordinated Science Laboratory,\\
			University of Illinois at Urbana-Champaign Urbana, IL 61801 USA}

\icmlkeywords{boring formatting information, machine learning, ICML}

\vskip 0.3in
]

\begin{abstract}
%We study the problem of learning a causal structure over a set of random variables when the experimenter is allowed to perform at most $M$ experiments in a non-adaptive manner. The non-adaptivity enables the experimenter to perform the experiments in parallel.
%We consider the problem of minimizing the unlearned portion of the graph with a limited number of experiments for both Bayesian and minimax cases.
 %%We assume that each experiment consists of a single intervention, which makes the method suitable for the applications in which simultaneous randomization of the variables is not feasible.
 %%The main question addressed in this work is that under the described condition, for a limited number of experiments, what is the maximum portions of the causal structure that can be learned.
 %After finding a theoretical solution for this problem, we propose the ProBal algorithm, which designs the experiments in a time efficient manner.
%We show that for bounded degree graphs, in the minimax case and in the Bayesian case with uniform prior, our proposed algorithm is a $\rho$-approximation algorithm, where $\rho$ is independent of the order of the underlying graph. Both synthesized and real data show that the performance of the ProBal algorithm is very close to the optimal solution.
We study the problem of causal structure learning over a set of random variables when the experimenter is allowed to perform at most $M$ experiments in a non-adaptive manner. We consider the optimal learning strategy in terms of minimizing the portions of the structure that remains unknown given the limited number of experiments in both Bayesian and minimax setting. We characterize the theoretical optimal solution and propose an algorithm, which designs the experiments efficiently in terms of time complexity. We show that for bounded degree graphs, in the minimax case and in the Bayesian case with uniform priors, our proposed algorithm is a  $\rho$-approximation algorithm, where $\rho$ is independent of the order of the underlying graph. Simulations on both synthetic and real data show that the performance of our algorithm is very close to the optimal solution.
\end{abstract}

\section{Introduction}
\label{sec:intro}

Causal structures are commonly represented by directed acyclic graphs (DAGs), where the vertices of the graph are random variables and a directed edge from $X$ to $Y$ indicates that variable $X$ is a direct cause of variable $Y$ \cite{pearl2009causality, spirtes2000causation, greenland1999causal}.
Given a set of variables, there are two main approaches for uncovering the causal relationships among them.
First to perform conditional independence tests on the set of variables based on observational measurements \cite{mooij2016distinguishing}.
The second involves intervening on some of variables to recover their causal effect on the rest of the variables.
Unlike observational only tests, sufficient experiments involving interventions can uniquely identify the underlying causal graph completely.

In the study of intervention-based inference approach,
often a setup in which the experimenter performs $M$ experiments on the set of variables is considered. In each experiment a set of at most $k$ variables are intervened on.
In this setting, two natural questions arise:
\vspace{-3mm}
\begin{enumerate}
\item What is the smallest required number of experiments in order to learn all the causal relations?
\item For a fixed number of experiments, what portion of the causal relationships can be learned?
\end{enumerate}
\vspace{-3mm}
\begin{figure}[t]
\vskip 0.02in
\begin{center}
\centerline{\includegraphics[scale=0.234]{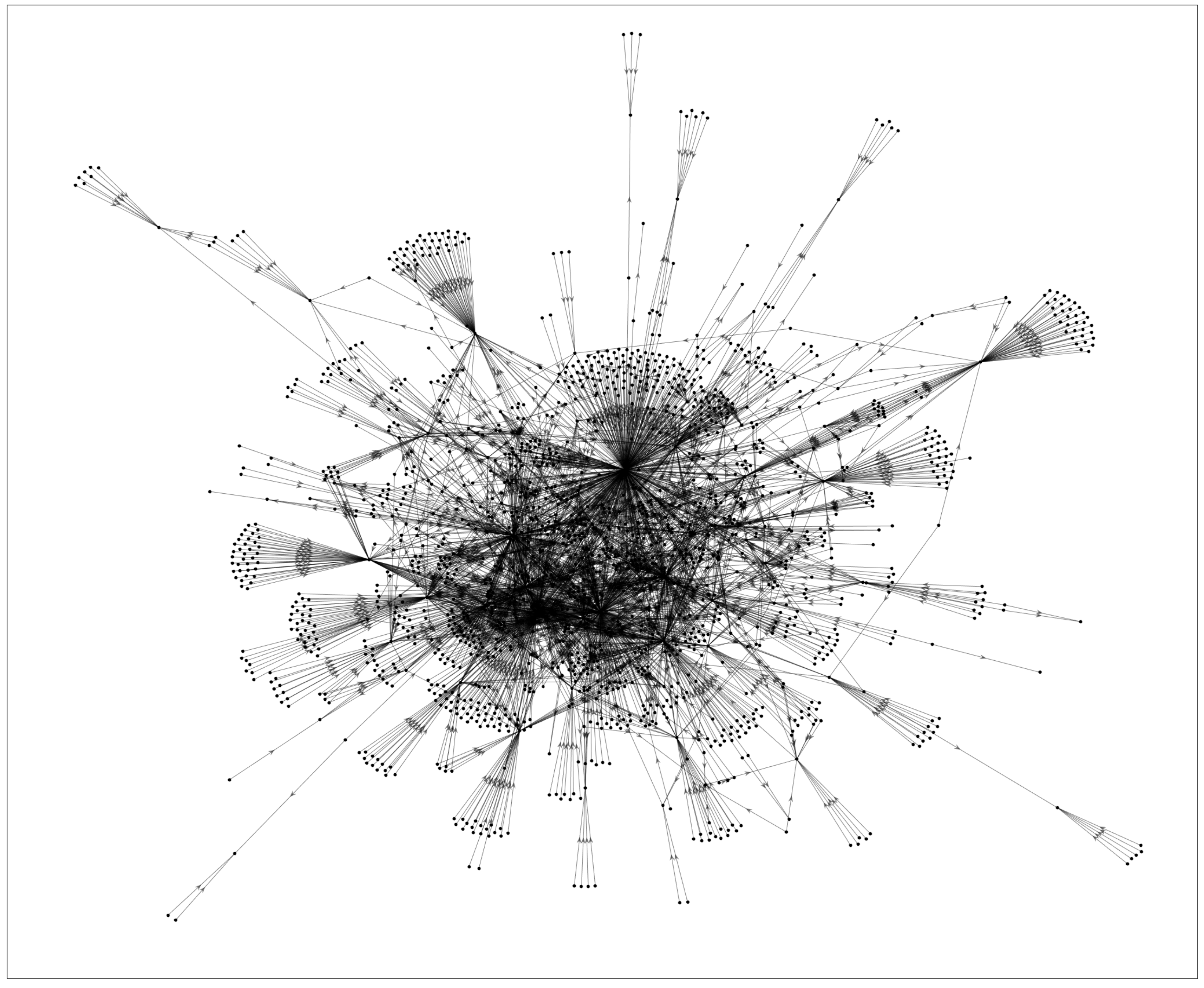}}
\caption{GRN for E-coli bacteria.}
\label{fig:EcoliGRN}
\end{center}
\vskip -0.34in
\end{figure}
The first problem has been addressed in the literature under different assumptions (see the Related Work subsection).
To the best of our knowledge, the second question has not been studied in the literature, and it is this question that we address herein.
Specifically, we consider a setup with $M$ experiments, each containing exactly one intervention.
The reason we consider single-intervention experiments is that in many applications, such as some experiments in biology, simultaneous intervention in multiple variables may not be feasible.
For the underlying structure on the variables, we assume that the number of cycles of length three in the structure is negligible compared to the order of the graph. Structures satisfying this assumption arise in many applications.
For instance, causal structure in gene regulatory network (GRN) for some bacteria such as the Escherichia coli (E-coli) and Saccharomyces cerevisiae (S. cerevisiae) has a tree like structure (see Figure \ref{fig:EcoliGRN}), and hence, satisfies our assumption \cite{cerulo2010learning}.

\textbf{Contributions.}
Unlike most of the previous work, we utilize a hybrid inference scheme instead of an adaptive approach (Subsection \ref{subsec:pre}).
In this approach, first an observational test, such as the IC algorithm \cite{pearl2009causality}, is performed on the set of variables. This test reveals the skeleton as well as the orientation of some of the edges of the causal graph. Next, based on the result of the initial test, the complete set of $M$ experiments is designed in a non-adaptive manner. Having the complete set of experiments, enables the experimenter to perform the interventional experiments in parallel.
The formal description of the problem of interest is provided in Subsection \ref{subsec:probdesc}.
We study the problem of structure learning for both Bayesian and minimax settings.
In Section \ref{sec:optsol}, we present the optimal solution for both settings. This solution is optimal in the sense of recovering the structure that minimizes the loss for a given number of interventional experiments. Finding this optimal solution is in general computationally intense. In Section \ref{sec:alg}, we propose ProBal algorithm, which finds the set of experiments in a computationally efficient manner. We show that for bounded degree graphs, in the minimax setting and in the Bayesian settings with uniform prior, our proposed algorithm is a $\rho$-approximation algorithm, where $\rho$ is independent of the order of the underlying graph.
In Section \ref{sec:expres}, using synthetic and real data, we show that the performance of ProBal is very close to the optimal solution.

\textbf{Related Work.}
%As mentioned earlier, there are two methods to distinguish the causal relationships among a set of variables.
The best known algorithms for general purely observational recovery approaches are IC \cite{judea1991equivalence} and PC \cite{spirtes1991algorithm}.
Such purely observational approaches reconstruct the causal graph up to Markov equivalence classes, and hence, the direction of some of the edges may remain unresolved.
Of course under some conditions, full causal structure learning using merely observational data is feasible \cite{shimizu2006linear, hoyer2009nonlinear, peters2012identifiability}.

There is a large body of research on learning causal structures using interventional data \cite{pearl2009causality, spirtes2000causation, woodward2005making, eberhardt2007causation, hauser2014two, cooper1999causal, he2008active}. Specifically,
Pearl \cite{pearl2009causality}, considers the SEM model and defines so-called do-intervention to infer the causal relations among a set of variables.
%One can decide whether the causal structure is identifiable in a systematic way using an algebraic method which is known as the do-calculus \cite{pearl1995causal}.
A similar approach for representing interventions is adopted in \cite{spirtes2000causation}, but it allows interventions to have non-degenerate distributions.
Woodward \cite{woodward2005making} proposed another type of intervention which unlike Pearl's, does not depend on a specified model of the causal relations among the random variables.
Peters et al. introduced invariant causal prediction \cite{peters2016causal}, which is a causal discovery method that uses different experimental settings to predict the set of ancestors of a variable.
In that work, data comes from different unknown experimental settings (which could results from interventions).
See \cite{meinshausen2016methods} for some validations of this method.

Regarding the first question discussed earlier, \cite{eberhardt2005number} consider the complete graph as the underlying causal structure to obtain the worst case bounds on the number of required experiments.
In that work, both cases of experiments containing bounded and unbounded number of interventions are studied.
In \cite{eberhardt2007causation,hyttinen2013experiment}, it has been shown that there is  a connection between the problem of finding a separating system in a graph and designing a proper set of experiments for causal inference, and hence, results from combinatorics help in finding the fundamental bounds.
%\todo[inline]{\cite{hauser2014two}}
In \cite{hauser2014two}, two algorithms that minimize the number of experiments in the worst case are developed. The proposed algorithms are adaptive and in the one with polynomial complexity, the size of experiments can be as large as half the order of the graph, which may not be practical in many real-life applications.
% hausser is adaptive?
In \cite{shanmugam2015learning}, the authors present information-theoretic lower bounds on the number of required experiments for both deterministic and randomized adaptive approaches. They also proposed an adaptive algorithm that allows to learn chordal graphs completely.
%Also, an algorithm for constructing a separating system and a deterministic and adaptive algorithm for fully learning chordal graphs are proposed in this work.
\vspace{-3mm}
\section{Model Description}
\label{sec:moddesc}

\subsection{Preliminaries}
\label{subsec:pre}
In this subsection we introduce some definitions and concepts that we require later.

\begin{definition}
Consider a directed graph $D=(V,E)$ with vertex set $V$ and set of directed edges $E$.
$D$ is a DAG if it is a finite graph with no directed cycles.
\end{definition}

\begin{definition}
A DAG $D$ is called causal if its vertices represent random variables $V=\{X_1, ...,X_n\}$ and a directed edges $(X_i,X_j)$ indicates that variable $X_i$ is a direct cause of variable $X_j$.
\end{definition}
We consider a structural equation model \cite{pearl2009causality}, which is a collection of $n$ equations $X_i=f_i(PA_{X_i},N_i)$, $i=1,...,n$, where $PA_{X_i}$ denotes the parents of $X_i$ in $D$, and $N_i$'s are jointly independent noise variables. We assume here that in our network, all variables are observable.
Also, throughout the rest of the paper, we assume the faithfulness assumption on the probability distribution.

%\begin{definition}
%The skeleton of a graph $D$ over the set of variables $V$ is an undirected graph over $V$ that contains an edge $\{X,Y\}$ for every directed edge $(X,Y)$ in $D$.
%\end{definition}

%\begin{definition}
%A distribution $P$ over the set of variables $\{X_1, ...,X_n\}$ is faithful to the causal DAG $D$ if $D$ represents all the independency relations contained in $P$.
%\end{definition}

\begin{definition}
Two causal DAGs $D_1$ and $D_2$ over $V$ are Markov equivalent if every  distribution that is compatible with one of the graphs is also compatible with the other. The set of all graphs over $V$ is partitioned into a set of mutually exclusive and exhaustive Markov equivalence classes, which are the set of equivalence classes induced by the Markov equivalence relation \cite{koller2009probabilistic}.
\end{definition}

\begin{definition}
A $v$-structure is a structure containing two converging directed edges whose tails are not connected by an edge. $v$-structures are also known as immorality and a graph with no immorality is called a moral graph.
\end{definition}

%In order to identify a causal DAG, one can utilize the sample values obtained either from pure observation or from interventional experiments.
%The former involves passive observation of random samples of the set of random variables (referred to which as the null experiment by \cite{eberhardt2005number}), while in the latter, the experimenter can actively intervene in the system and randomize the value of some variables.
%It is well known that using only data from an observational study, one can distinguish the causal relationships up to Markov equivalence.
Using purely observational data  (referred to which as the null experiment by \cite{eberhardt2005number}), one can utilize  a ``complete" conditional independence based algorithm to learn the causal structure as much as possible. By complete we mean that the algorithm is capable of distinguishing all the orientations up to the Markov equivalence class. Such an algorithm includes performing a conditional independence test followed by applying the Meek rules \cite{pearl2009causality}.
%The structure obtained from these algorithms has the property that all the $v$-structures in the graph, as well as the direction of edges which prevent extra $v$-structures or directed cycles are recovered.
On the other hand, interventions can enable us to differentiate among the different causal structures inside a Markov equivalence class. Define an intervention $I$ on variable $X\in V(D)$ as removing the influence of all the variables on $X$ and randomizing the value of this variable. We denote this intervention by $I=X$. An inference algorithm consists of a set of $M$ experiments\footnote{Note  that in most of the other work in this area, each intervention is what we refer to as experiment here and hence, each intervention can contain as many as $n$ variable randomization.} $\mathcal{E}_{total}=\{\mathcal{E}_1, \mathcal{E}_2, ..., \mathcal{E}_M\}$, where each experiment contains $k$ interventions, i.e., $\mathcal{E}_i=\{I_1^{(i)}, I_2^{(i)}, ..., I_k^{(i)}\}$ for $1\le i \le M$. As shown in \cite{eberhardt2007causation}, observing the result of the null experiment,  one can find the orientation of the edge between any two variables $X_i$ and $X_j$, if there exists $\mathcal{E}_k\in\mathcal{E}_{total}$ such that
$
(X_i\in\mathcal{E}_k\textit{ , }X_j\notin\mathcal{E}_k)\text{ or }
(X_j\in\mathcal{E}_k\textit{ , }X_i\notin\mathcal{E}_k).
$

An inference algorithm may be (1) adaptive, in which case the experiments are performed sequentially and the information obtained from the previous experiments is used to design the next one; (2) passive, in which all the experiments are designed beforehand;
(3) hybrid, in which the experimenter first performs a pure observational study to obtain the skeleton and some of the orientations in the causal DAG, and then designs the rest of the experiments in a passive manner.
The third approach is referred to as the passive setup by \cite{shanmugam2015learning}, while \cite{eberhardt2005number} use the term passive for a setting in which the interventions are selected without performing the null experiment.

\vspace{-0mm}
\subsection{Problem Definition}
\label{subsec:probdesc}

Since in many practical applications, it is resource-intensive to perform the experiments sequentially, we will investigate the hybrid approach for the design of experiments. 
A nice feature of the hybrid approach is that it allows us to parallelize the performance of the experiments.
%This is due to the fact that one can perform all the experiments in a hybrid approach in parallel on separate setups which have the same distribution. 
For example, in the study of GRNs introduced in Section \ref{sec:intro}, the GRN of all E-coli cells are the same and experiments can be performed on different cells simultaneously. 
%As mentioned earlier, in this approach, first a purely observational test will be performed on the system. The statistical data obtained form this step will allow us to obtain the full skeleton and find the orientation of some of the edges up to Markov equivalence using a complete conditional independence based algorithm.

\begin{definition}
A chord of a cycle is an edge not in the cycle whose endpoints are in the cycle. A hole in a graph is a cycle of length at least 4 having no chord. A graph is chordal if it has no hole.
\end{definition}

\begin{figure}[t]
\vskip 0.1in
\begin{center}
\centerline{\includegraphics[scale=0.62]{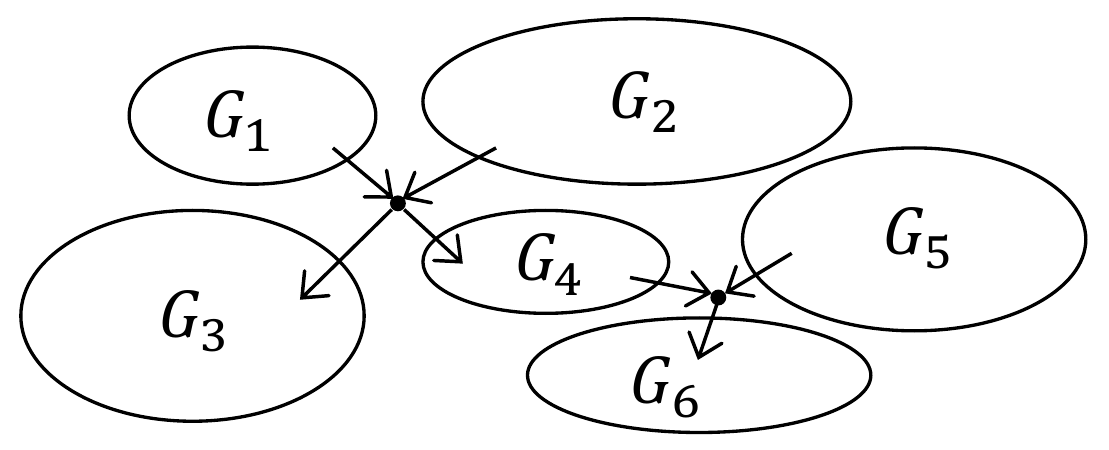}}
\caption{Example of a DAG, in which the procedure results in six disjoint chordal moral graphs. }
\label{fig:SampleD}
\end{center}
\vskip -0.2in
\end{figure}

Let $Z$ be the set of edges whose orientations are identified after the obseravtional test. Consider the moral graph $D\backslash Z$. As noted in \cite{shanmugam2015learning}, $D\backslash Z$ consists of a set of disjoint chordal moral graphs $\{G_1, G_2, ..., G_K\}$ (see Figure \ref{fig:SampleD} as an example). To learn the structure, it suffices to learn $\{G_1, G_2, ..., G_K\}$. We focus on one such graph, say $G$. Hence, in the remainder of the paper, $G$ is assumed to be a moral chordal DAG.

\begin{definition}
Let $S$ be the set of all edges in a chordal graph which do not belong to any triangle.
We refer to a connected component of $G\backslash S$ of order larger than one as ``cyst''.
\end{definition}

\begin{assumption}
\label{ass:1}
We assume that the number of triangles, compared to the order of the chordal graph is negligible. Formally, we assume that the proportion of the number of triangles to the order of $G$ goes to zero as $n=|V(G)|$ tends to infinity.
\end{assumption}

Assumption \ref{ass:1} implies that both the number of cysts and their orders are negligible.
%We refers to a structure with negligible number of triangles, as a semi-triangle-free DAG.
As mentioned in Section \ref{sec:intro}, structures with this property occur frequently in genomic and other applications.

As mentioned earlier, in some scenarios the experimenter may be restricted to perform a limited number of experiments.
Hence, we focus on the following problem: \textit{If the experimenter is allowed to perform $M$ experiments, each of size $k=1$. What portion of the graph could be reconstructed on average and in the worst case?}\\ We shall distribute our total budget of $M$ experiments over the components $\{G_1, G_2, ..., G_K\}$ proportionally to their size. Thus, we assume that we are capable of performing $m$ experiments $\mathcal{E}=\{\mathcal{E}_1, \mathcal{E}_2, ..., \mathcal{E}_m\}\subseteq\mathcal{E}_{total}$ on component $G$.

\begin{figure}[t]
\vskip 0.1in
\begin{center}
\centerline{\includegraphics[scale=0.38]{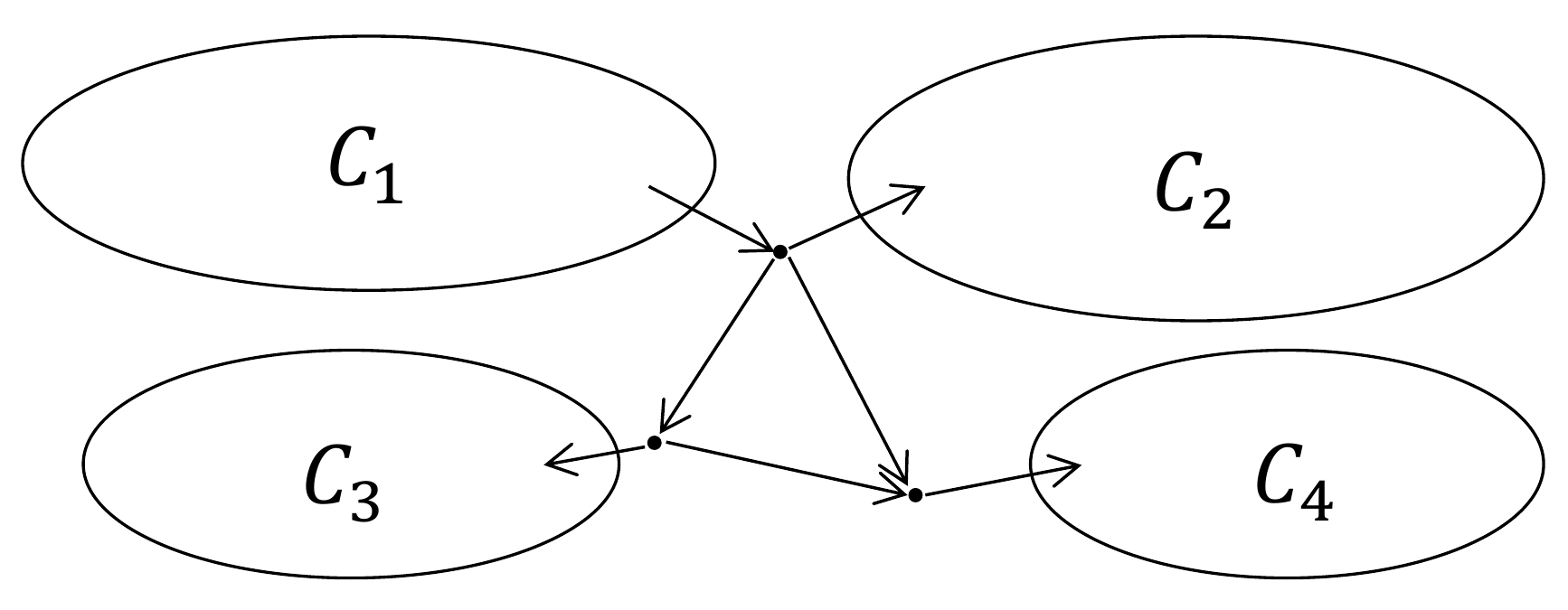}}
\caption{Example of a cyst in which intervening on all vertices is necessary. In this figure, the edges in component $C_2$ will not be identified unless we intervene on the top vertex.}
\label{fig:worstcyst}
\end{center}
\vskip -0.2in
\end{figure}

Consider the graph obtained by contracting all the vertices of a cyst into a single vertex. Clearly, this merging graph would be a tree,
%We consider each cyst as vertex by virtue of contracting all the vetices of the cyst on one vertex. Clearly, the resulting graph will be a tree, 
and hence, we work on a moral tree structure after the contractions. If we are required to intervene on a vertex in the resulting tree that corresponds to a cyst, we will intervene on all the vertices of that cyst. This is in some cases necessary as shown in Figure \ref{fig:worstcyst}. Note that after intervening in all the nodes of a cyst, we recover all the orientations inside the cyst. We emphasize that some edge orientations remain unidentified in the cysts that we have not intervened on; however, due to Assumption \ref{ass:1}, the number of such edges will not scale with $n$. Also, due to Assumption \ref{ass:1}, we can count intervening on a contracted cyst as one intervention without impacting the scaling results.

\begin{definition}
Root variable is a variable for which the number of edges entering that variable is zero.
\end{definition}

\begin{lemma}
\label{lem:1root}
A moral chordal graph $G$ has only one root unless they all belong to the same cyst.
\end{lemma}
\begin{proof}
It suffices to prove that we cannot have two roots in two separate cysts. We prove this by contradiction.
Suppose roots $r_1$ and $r_2$ exist in two separate cysts. Since the component is assumed to be connected, there exists a vertex $v$ where a directed path from $r_1$ to $v$ meets a directed path from $r_2$ in $v$. Let the parents of $v$ on these paths  be $p_1$ and $p_2$, respectively. If $p_1$ and $p_2$ are not adjacent, we will have a $v$-structure on the set $\{p_1,v,p_2\}$ which is a contradiction. Otherwise, $p_1$ and $p_2$ are adjacent and $\{p_1,v,p_2\}$ are in a cyst. Without loss of generality, assume $p_1$ is a parent of $p_2$. Let $p_3$ be the parent of $p_2$ on the path from $r_2$. In this case, we will have the same argument for the set  $\{p_1,p_2,p_3\}$, and we either have a $v$-structure, or $p_3$ is also in the cyst containing $v$. Continuing this argument, we either have a $v$-structure or $r_1$ and $r_2$ are in the same cyst as $v$, which is a contradiction.
\end{proof}

\begin{lemma}
\label{lem:root}
In a moral chordal graph, randomizing the root suffices to fully learn the orientations of all the edges not belonging to any cysts.
\end{lemma}
\begin{proof}
Let $r$ be the root vertex. Consider an edge $e=(u,v)$.
Since $e$ does not belong to any directed cycle, and there is a directed path from $r$ to $u$ and also a directed path from $r$ to $v$, $e$ should be directed from $u$ to $v$, otherwise it contradicts Lemma \ref{lem:1root}.
\end{proof}

Therefore, in the sequel, we focus on a moral tree structure and from Lemma \ref{lem:root}, we know that randomizing the root variable will give us all the orientations.

We model this problem as follows. Let $\mathcal{P}=\{P^{\theta}:\theta\in\Theta\}$ be the set of probability distributions over the location of the root of the tree. We assume that the probability distributions of interest are all positive.

The following remark is a consequence of Lemma \ref{lem:root}.
\begin{remark}
\label{rmk:root}
For each location of the root, only one moral tree is consistent with the skeleton of the tree. Therefore, $\mathcal{P}$ is the set of probably distributions over the realizations of the moral tree.
\end{remark}

Define $U(\mathcal{E})$ as the set of edges whose orientation is not found after performing the experiment set $\mathcal{E}$. For the given skeleton obtained from the observational test, let $T_v$ be the moral tree of order $n$ with vertex $v$ as its root.
We define the loss of an experiment set $\mathcal{E}$ on $T_v$ as
$l(\mathcal{E},T_v)=|U(\mathcal{E})|$,
and the average loss of the experiment set under distribution $P^{\theta}$ as
\[
L_{\theta}(\mathcal{E})=\sum_{v}P^{\theta}(v)l(\mathcal{E},T_v).
\]

The problem of finding the best experiment set for the worst case, could be stated as the following minimax problem:
\begin{equation}
\begin{aligned}
&\min_{\mathcal{E}}\max_{\theta}L_{\theta}(\mathcal{E}),\\
s.t.~|\mathcal{E}|=m,
&\hspace{3mm}|\mathcal{E}_i|=1~~~\forall i:~1\le i\le m
\end{aligned}
\end{equation}
In some real life applications, the experimenter may have prior knowledge about the possible location of the root in the tree. That is, the probability distribution $P^{\theta}\in\mathcal{P}$ over the location of the root in the tree is known.  
%Hence, there is no need to design the algorithm based on the worst distribution. 
In such a setting, we can investigate the following Bayesian version of the problem:
%Defining $L(\mathcal{E})=\sum_{v\in V(T)}P_r(v)l(\mathcal{E},T_v)$, we consider the following optimization problem:
\vspace{-2mm}
\begin{equation}
\begin{aligned}
&\min_{\mathcal{E}}L_{\theta}(\mathcal{E}),\\
s.t.~|\mathcal{E}|=m,
&\hspace{3mm}|\mathcal{E}_i|=1~~~\forall i:~1\le i\le m
\end{aligned}
\end{equation}

\section{Optimal Solution}
\label{sec:optsol}

Consider a moral tree $T$ of order $n$.
%As mentioned in Section \ref{sec:moddesc},
In this structure, every non-root vertex has incoming degree $d^-(v)=1$,
the root vertex, $r$ has incoming degree $d^-(r)=0$, and recall that intervening on the root identifies the whole tree.

Assume the experiment set $\mathcal{E}=\{\{I^{(1)}_1\},...,\{I^{(m)}_1\}\}$ was performed on the moral tree $T$. Let $\mathcal{I}=\{I^{(1)}_1,...,I^{(m)}_1\}$ be the set of variables on which we intervened. In this case, the subgraph $T\backslash\mathcal{I}$ will be a forest containing $J$ components $\{C_1,...,C_J\}$.

\begin{lemma}
\label{lem:loss}
Performing experiment set $\mathcal{E}$, on $T_v$ rooted at $v$,
\begin{equation*}
l(\mathcal{E},T_v) = \left\{
\begin{array}{l l}
0 & \quad v\in\mathcal{I},\\
|C_j|-|B_j| & \quad v\in C_j,
\end{array} \right.
\end{equation*}
where $C_j\in \{C_1,...,C_J\}$, and $B_j=C_j\cap N(\mathcal{I})$, where, $N(\mathcal{I})$ denotes the set of neighbors of variables in $\mathcal{I}$.
\end{lemma}
%Note that $1\le|B_i|\le m$.
\vspace{-5mm}
\begin{proof}
After performing a single intervention $I=X$, the orientation of all the edges entering the descendants of $X$ and the edges entering $X$ itself will be recovered. Thus
$
l(X,T_v)=|\textit{ND}_v(X)|-1_{\{X\neq v\}}.
$
As a result, intervening on variables in $\mathcal{I}$, if $v\in\mathcal{I}$, since $|\textit{ND}_v(v)|=0$ the loss would be equal to zero; otherwise, if $v\in C_j$ we will have:\\
$l(\mathcal{E},T_v)=|\bigcap_{X\in\mathcal{I}}\textit{ND}_v(X)
|-|\bigcap_{X\in\mathcal{I}}\textit{ND}_v(X)\cap N(\mathcal{I})|
$\\$=|C_j|-|B_j|$.
%\begin{align*}
%l(\mathcal{E},T_v)&=|\bigcap_{X\in\mathcal{I}}\textit{ND}_v(X)
%|-|\bigcap_{X\in\mathcal{I}}\textit{ND}_v(X)\cap N(\mathcal{I})|\\
%&=|C_j|-|B_j|.
%\end{align*}
\end{proof}

\begin{theorem}
\label{thm:loss}
%Let $U\in\mathcal{P}$ be the uniform distribution of the location of the root over vertices of $T$.
The average loss could be bounded as follows
\vspace{-5mm}
\begin{equation}
\label{eq:thm1}
\sum_{j=1}^JP^{\theta}(C_j)|C_j|-m\le L_{\theta}(\mathcal{E})\le\sum_{j=1}^JP^{\theta}(C_j)|C_j|-1,
\end{equation}
%where $\epsilon=m$.
\end{theorem}
\vspace{-5mm}
\begin{proof}
Using Lemma \ref{lem:loss}, we have
\vspace{-1mm}
\begin{align*}
L_{\theta}&(\mathcal{E})=\sum_v P^{\theta}(v)l(\mathcal{E},T_v)\\
&=\sum_{v\in\mathcal{I}}P^{\theta}(v)\times0
+\sum_{j=1}^J\sum_{v\in C_j}P^{\theta}(v)(|C_j|-|B_j|)\\
&=\sum_{j=1}^J(|C_j|-|B_j|)\sum_{v\in C_j}P^{\theta}(v)\\
%&=(1+\frac{1}{n-1})\sum_{j=1}^J\frac{1}{n}(|C_j|-|B_j|)P_{\theta}(C_j)\\
%&=\sum_{j=1}^JP_{\theta}(C_j)U(C_j)-\sum_{j=1}^JP_{\theta}(C_j)U(B_j)\\
&=\sum_{j=1}^JP^{\theta}(C_j)|C_j|
-\sum_{j=1}^JP^{\theta}(C_j)|B_j|.
\end{align*}
Note that for all $j$, $1\le|B_j|\le m$, and the result is immediate.
%In the last expression, denoting the second, terms by $A$, we have $-{m}/{n}\le A\le-{1}/{n}$. This implies the desired result.
\end{proof}
\vspace{-3mm}
Since we have assumed that $m\ll n$, in the sequel we will focus on minimizing
\[
 \hat{L}_{\theta}(\mathcal{E})=\sum_{j=1}^JP^{\theta}(C_j)|C_j|.
\]
\begin{corollary}
\label{cor:unif}
In the special case of uniform $P^{\theta}$, the function $\hat{L}$ could be obtained as
$\hat{L}_{\theta}(\mathcal{E})=\frac{1}{n}\sum_{j=1}^J|C_j|^2$.
\end{corollary}

%\subsection{Bayesian Approach}
%\label{subsec:optsolbayes}

\textbf{Bayesian Setting.} In the Bayesian setting, we seek the set of experiments that minimizes $\hat{L}_{\theta}$. This can be done by checking all $\binom{n}{m}$ possible vertex sets of size $m$. This brute-force solution is computationally intensive.
%For instance, if in a graph of order $n=100$, and $m=10$ experiments, $\binom{100}{10}\ge10^{10}$ options should be checked, which may not be feasible.
In Section \ref{sec:alg} we will introduce an efficient approximation algorithm instead.

%\subsection{Minimax Approach}
%\label{subsec:mM}

\textbf{Minimax Setting.} As mentioned in Section \ref{sec:moddesc}, in the minimax setting we are interested in finding optimal $\mathcal{E}$ in $\min_{\mathcal{E}}\max_{\theta}\hat{L}_{\theta}(\mathcal{E})$. Note that the loss is maximized if all the probability mass of root is put on the largest component. That is
\vspace{-4mm}
\begin{equation}
\label{eq:mMcomp}
\begin{aligned}
\min_{\mathcal{E}}\max_{\theta}\hat{L}_{\theta}(\mathcal{E})
&=\min_{\mathcal{E}}\max_{\theta}\sum_{j=1}^JP^{\theta}(C_j)|C_j|\\
&\le\min_{\mathcal{E}}\max_{j}|C_j|.
\end{aligned}
\vspace{-3mm}
\end{equation}

Therefore, it suffices to choose a variable set $\mathcal{I}=\{X_1,...,X_m\}$, which minimizes $\max_j |C_j|$.\\
We can again use a brute-force solution which suffers from the same computational complexity issues as in the Bayesian case. We will propose an approximation algorithm for the minimax approach in Section \ref{sec:alg}.

\section{Efficient Learning Algorithm}
\label{sec:alg}

%\subsection{Description}

In this section, we propose an algorithm for finding experiment sets efficiently for both Bayesian and minimax settings. First, we define the concept of a separator vertex, which plays a key role in the proposed algorithm.
We shall see that in our method, we require the prior on the location of the root variable.
Using separators, we propose the probability balancer (ProBal) algorithm, which allows for experiment design. The main idea is to iteratively decompose the tree into two subtrees, referred to as segments, sharing a separator vertex. In Subsection \ref{subsec:mmex}, we consider the minimax extension in which no prior is not available to the experimenter.
%For this case we propose a modified version of the ProBal algorithm.

\begin{definition}
Let lobes of a vertex $v$ in a tree be the remaining components in the graph after removing $v$.
A vertex $v$  in a tree $T$ is a separator if the probability of the root being in each of its lobes is less than $1/2$.
\end{definition}
\begin{proposition}
\label{prop:exist}
There exists a separator vertex in any tree.
\end{proposition}
\vspace{-5mm}
\begin{proof}
Consider any arbitrary vertex $v_1$ in the tree. If all its lobes have probability less than $1/2$, then it is a separator; otherwise, only one of its lobes, say lobe $b$, has probability larger than or equal to $1/2$. Consider $v_2$, which is the neighbor of $v_1$ in $b$. Since the probabilities should add up to 1, the lobe connected to $v_2$ through $v_1$ should have probability less than $1/2$ and hence we continue with checking the probability of the other lobes of $v_2$. This process will result in finding a separator because there are no cycles in a tree and we have assumed that the tree is finite.
\end{proof}

\begin{proposition}
\label{prop:2/3}
The lobes of a separator vertex can be partitioned into two wings such that the probability of the root being in each of them is less than $2/3$.
\end{proposition}
\vspace{-5mm}
\begin{proof}
Suppose $v$ is a separator vertex with lobes $b_1,\cdots,b_l$ sorted in ascending order of their probabilities  $P^{\theta}(b_i)=p_i$. We also add lobe $b_0$ with $p_0=0$.
Let $j$ be the largest index such that $\sum_{i=0}^jp_i\le2/3$.
If $j=l$, any arbitrary partitioning of the lobes is acceptable. Assume $j<l$.
 If $\sum_{i=j+1}^lp_i\le2/3$, then $\{\{b_0,\cdots,b_j\},\{b_{j+1},\cdots,b_l\}\}$ is the desired partitioning. Otherwise, we have $\sum_{i=0}^{j+1}p_i>2/3$ and $\sum_{i=j+1}^lp_i>2/3$, which implies that $p_{j+1}>1/3$ and since $v$ is a separator, $p_{j+1}<1/2$. Therefore, $P^{\theta}(\{b_0,\cdots b_{l}\}\backslash \{b_{j+1}\})<2/3$, and $\{\{b_{j+1}\},\{\{b_0,\cdots b_{l}\}\backslash \{b_{j+1}\}\}\}$ is the desired partitioning.
\end{proof}
\vspace{-3mm}

ProBal algorithm searches for $m$ variables in a given tree $T$ in an iterative manner.
It starts with the original tree as the initial segment and in each round it breaks it into smaller segments in the following manner. Let $\mathcal{G}$ be the set containing all the segments.
At each round the algorithm picks the segment $G_m$ with largest $P^{\theta}(G)$ in the set $\mathcal{G}$, and finds the most suitable separator (described below) and adds it to the intervention set $\mathcal{I}$ (if it is not already in $\mathcal{I}$).
This is done using the function FindSep.
 Then using the function Div, the algorithm divides $G_m$ into two new segments $G_1$ and $G_2$, and replaces $G_m$ with $\{G_1,G_2\}$ in the set $\mathcal{G}$ unless they have a star structure with the used separator as the center vertex. The reason that we ignore the star structures is that since the center is already chosen to be intervened on, the orientation of all the edges of the structure are discovered and there is no need for further interventions on the other variables in the star. The process of choosing separators continues until $m$ variables are collected or all the graph is resolved. The set $\mathcal{I}$ will be returned as the set of intervention variables.
The functions FindSep$(\cdot)$ and Div$(\cdot)$ are described below.
\vspace{-3mm}
 \begin{itemize}
\item In the function FindSep$(\cdot)$, first we normalize the values of probability in the input segment to obtain distribution $\overline{P^{\theta}}$. For any variable $X\in G_m$, we compute the probability of root being in the lobes of $X$, and partition the lobes into two wings $W_1^*(X)$ and $W_2^*(X)$ such that the probability of the root being in the wings is as balanced as possible. Define the unbalancedness of $X$ as $s(X)\coloneqq |\overline{P^{\theta}}(W_1^*(X))-1/2\overline{P^{\theta}}(W_1^*(X)\cup W_2^*(X))|$. The function returns variables $X^*$ with minimum $s(X)$.
\item The function Div$(\cdot)$ outputs segments $G_1=G_m\backslash W_2^*(X^*)$ and $G_2=G_m\backslash W_1^*(X^*)$. In both segments, it sets the probability of $X^*$ to zero.
\end{itemize}
\vspace{-3mm}
\begin{algorithm}[t]
\begin{algorithmic}
 \STATE {\bf Input:} $T$, $m$.
\STATE $\mathcal{I}=\emptyset$, $\mathcal{G}=\{T\}$.
\WHILE{$|\mathcal{I}|<m$ \AND $\mathcal{G}\neq\emptyset$}
\STATE $G_m=\arg\max_{G\in\mathcal{G}}P^{\theta}(G)$.
\STATE $X^*=\text{FindSep}(G_m)$.
\STATE $\mathcal{I}=\mathcal{I}\cup X^*$.
\STATE $(G_1,G_2)=\text{Div}(G_m,X^*)$.
\STATE $G_{new}=\{G_1,G_2\}$.
\STATE For $i\in\{1,2\}$, if $G_i$ is a star graph with center vertex  $X^*$, or if $V(G_i)\subseteq \mathcal{I}$, remove $G_i$ from $G_{new}$.
\STATE $\mathcal{G}=\mathcal{G}\backslash G_m$.
\STATE $\mathcal{G}=\mathcal{G}\cup G_{new}$.
 \ENDWHILE
\STATE {\bf Output:} $\mathcal{I}$.
 \caption{ProBal Algorithm}
 \label{alg:ProBal}
\end{algorithmic}
\end{algorithm}

\begin{lemma}
\label{lem:alg}
(a) A leaf will not be chosen as the separator more than once.\\
(b) If the chosen separator $X^*$ in $G_m$ is not a leaf, and the segments $G_1$ and $G_2$ are produced, then
$
\max\{|V(G_1)|,|V(G_2)|\}<|V(G_m)|.
$
\end{lemma}
\vspace{-5mm}
\begin{proof}
(a) Suppose a leaf variable $X$ is chosen as the separator in one of the rounds of the algorithm. Consequently, its probability will be set to zero in the segment containing it. Since $X$ is a leaf, one of its wings will be empty, and hence after normalization of the probability in function FindSep$(\cdot)$, the wings of $X$ will have probabilities 0 and 1, while any other variable $X'$ in the segment with non-zero probability, will have wings that are balancing the measure $1-\overline{P^{\theta}}(X')$ (Note that all the variables in this segment cannot have zero probability, because the distribution is assumed to be positive and also all the other variables could not have been picked as separators before, otherwise the algorithm would not have kept this segment). Therefore, the function FindSep$(\cdot)$ will not choose $X$.\\
(b) Since $X^*$ is not a leaf, $|V(W_2^*(X^*))|\neq 0$, and since $|V(G_1)|=|V(G_m)|-|V(W_2^*(X^*))|$, we have $|V(G_1)|<|V(G_m)|$. Similarly $|V(G_2)|<|V(G_m)|$.
\end{proof}
\vspace{-3mm}
\begin{theorem}
ProBal algorithm runs in time $O(n^3)$.
\end{theorem}
\vspace{-5mm}
\begin{proof}
One can find the probability of each lobe of a vertex in linear time by running Depth-first search (DFS) algorithm. Therefore, FindSep$(\cdot)$ runs in time $O(n^2)$. Additionally, Div$(\cdot)$ runs in $o(n^2)$. From Lemma \ref{lem:alg}, ProBal algorithm will end in at most $n$ rounds. Thus, the time complexity of the algorithm is $O(n^3)$.
\end{proof}

\vspace{-3mm}
\subsection{Analysis}

In this subsection we find bounds on the performance of ProBal algorithm. We will show that in the case of a uniform prior, the proposed algorithm is a $\rho$-approximation algorithm, where $\rho$ is independent of the order of the graph.
\vspace{-3mm}
\begin{theorem}
\label{thm:UP1}
%Using the experiment set containing $m$ experiments obtained from the ProBal algorithm, the loss $\hat{L}_{\theta}(\mathcal{E})$ is upper bounded as follows:
After running ProBal algorithm for $r$ rounds and obtaining the experiment set $\mathcal{E}$, the loss $\hat{L}_{\theta}(\mathcal{E})$ is upper bounded as
\vspace{-3mm}
\[
\hat{L}_{\theta}(\mathcal{E})\le(\frac{2}{3})^{\lfloor \log_2(r+1) \rfloor}n.
\]
\vspace{-3mm}
\end{theorem}
\vspace{-3mm}
\begin{proof}
First we claim that with $r=2^k-1$, the set $\mathcal{G}$ defined in the algorithm which contain at most $2^k$ segments, has the property
$\max_{G\in\mathcal{G}}P^{\theta}(G)\le(\frac{2}{3})^k$.
We use induction to prove this claim. The base of the induction is clear from Proposition \ref{prop:2/3} and the fact that we set the probability of the separator itself to zero in function Div$(\cdot)$. For the induction step we need to show that after $r=2^{k+1}-1$ rounds $\max_{G\in\mathcal{G}}P^{\theta}(G)\le(\frac{2}{3})^{k+1}$.
By the induction hypothesis with $r=2^k-1$ rounds $\max_{G\in\mathcal{G}}P^{\theta}(G)\le(\frac{2}{3})^k$. Now, after the extra $2^k$ rounds, if a different segment were divided in each of those rounds, by Proposition \ref{prop:2/3}, the desired result could be concluded. Otherwise, at least one of the segments, say $G'$, was not divided while there exists another segment, say $G''$, which was divided more than once.
This implies that in the second dividing of $G''$, at least one of its sub-segments, say $G_1''$, obtained from the first division, had a larger probability than $P^{\theta}(G')$. But by Proposition \ref{prop:2/3}, we have $P^{\theta}(G_1'')\le\frac{2}{3}(\frac{2}{3})^{k}$. Therefore, $P^{\theta}(G')\le(\frac{2}{3})^{k+1}$.\\
Each component $C_j$ belongs to a segment $G\in\mathcal{G}$. Hence, by the claim above, for all $j$, $P^{\theta}(C_j)\le(\frac{2}{3})^k$. This concludes
%\begin{align*}
$\hat{L}_{\theta}(\mathcal{E})=\sum_{j=1}^JP^{\theta}(C_j)|C_j|
%&\le\sum_{j=1}^J(\frac{2}{3})^k|C_j|\\
\le(\frac{2}{3})^kn
\le(\frac{2}{3})^{\lfloor \log_2(r+1) \rfloor}n$.
%\end{align*}
\end{proof}
\vspace{-3mm}
In the following we prove that in the case of uniform prior on bounded-degree graphs,  ProBal is a $\rho$-approximation algorithm, where $\rho$ is independent of the order of the graph. To this end, we first obtain a lower bound on the loss of the optimum algorithm introduced in Section \ref{sec:optsol}.
\begin{lemma}
\label{lem:LB}
Consider tree $T$ of order $n$ with maximum degree $\Delta(T)$ and uniform probability distribution over the location of the root of the tree. %Let $\mathcal{E}^*=\arg\min_{\mathcal{E}}\hat{L}_{\theta}(\mathcal{E})$ in which we intervene on $m$ variables that minimize the average loss $\hat{L}_{\theta}$.
Then for any experiment set $\mathcal{E}$, we have
%\vspace{-2mm}
$
\frac{(n-m)^2}{n(\Delta(T)m-1)}\le\hat{L}_{\theta}(\mathcal{E}).
$
%\vspace{-7mm}
\end{lemma}
\vspace{-3mm}
\begin{proof}
For $J$ components, from Corollary \ref{cor:unif}, the loss function may be lower bounded as follows
%\begin{align*}
$\hat{L}_{\theta}(\mathcal{E})=\frac{1}{n}\sum_{j=1}^J|C_j|^2\ge \frac{1}{n}J(\frac{n-m}{J})^2.$
%\end{align*}
Since the tree is connected, the maximum number of components created by an experiments set of size $m$ is $\Delta(T)m-1$, which implies the result.
\end{proof}
\vspace{-3mm}
\begin{theorem}
Consider tree $T$ of order $n$ with maximum degree $\Delta(T)$.
In the case of uniform prior distribution, if $m\le\epsilon n$, ProBal is a $\rho$-approximation algorithm, where
$
\rho=\frac{\frac{3}{2}(m\wedge r)^{\log_2\frac{2}{3}}(\Delta(T)m-1)}{(1-\epsilon)^2},
$
where $m\wedge r \coloneqq \min\{r,m\}$.
%which is constant in $n$, and linear in $m$ and $\Delta(T)$.
$\rho$ is constant in $n$, polynomial of degree less than $0.42$ in $m$, and linear on $\Delta(T)$.
\end{theorem}
\vspace{-5mm}
\begin{proof}
Proof is immediate from Theorem \ref{thm:UP1} and Lemma \ref{lem:LB}.
%If $m\le r$, the proof is immediate from Theorem \ref{thm:UP1} and Lemma \ref{lem:LB}. In the case of $m>r$, all the segments have become stars, and hence all the components have been of order 1, which implies that $\hat{L}_{\theta}(\mathcal{E})\le 1$
\end{proof}
\vspace{-5mm}
\subsection{Minimax Setting}
\label{subsec:mmex}

As shown in \eqref{eq:mMcomp} in the minimax setting, we need to minimize the size if the largest component generated from the set of intervention variables.
If in ProBal, instead of balancing the probability of existence of the root variable in wings, we chose the separators to balance the number of vertices in the wings, then regardless of the probability distribution $P^{\theta}$, it would guarantee that in each round of the algorithm, the largest segment $G_m$  would be divided into two segments, where each of these segments would contain at most $\frac{2}{3}|V(G_m)|+1$ vertices. Therefore, at most the orientation of $\frac{2}{3}|V(G_m)|$ of the edges in $G_m$ may not be found.
This is equivalent to running ProBal algorithm on a uniform distribution $P^{\theta}$.

\begin{theorem}
\label{thm:mM}
Consider tree $T$ of order $n$ with maximum degree $\Delta(T)$. Let $\hat{L}_{W}(\mathcal{E})$ be the minimax loss of the experiment set $\mathcal{E}$.\\
(a)
After running ProBal for $r$ rounds and obtaining the experiment set $\mathcal{E}$, the loss $\hat{L}_{W}(\mathcal{E})$ is upper bounded as follows:
$
\hat{L}_{W}(\mathcal{E})\le(\frac{2}{3})^{\lfloor \log_2(r+1) \rfloor}n.
$\\
(b)
For any experiment set $\mathcal{E}$, we have
$
\frac{n-m}{\Delta(T)m-1}\le\hat{L}_{W}(\mathcal{E}).
$\\
(c)
If $m\le\epsilon n$, ProBal is a $\rho$-approximation algorithm, where
$
\rho=\frac{\frac{3}{2}(m\wedge r)^{\log_2\frac{2}{3}}(\Delta(T)m-1)}{1-\epsilon}.
$
\end{theorem}
\vspace{-4mm}
\begin{proof}
Since the performance of ProBal for the minimax case is the same as the Bayesian case with uniform distribution, the proof of part (a) follows from Theorem \ref{thm:UP1}.
For part (b), since the tree is connected, the maximum number of components created by an experiments set of size $m$ is $\Delta(T)m-1$, and we can minimize the order of the largest, by making the orders equal. That is, the order of the largest component is at least $\frac{n-m}{\Delta(T)m-1}$.
The proof of part (c) is immediate from parts (a) and (b).
\end{proof}

\begin{figure*}[!t]
\vskip 0.1in
\begin{center}
\centerline{\includegraphics[scale=0.45]{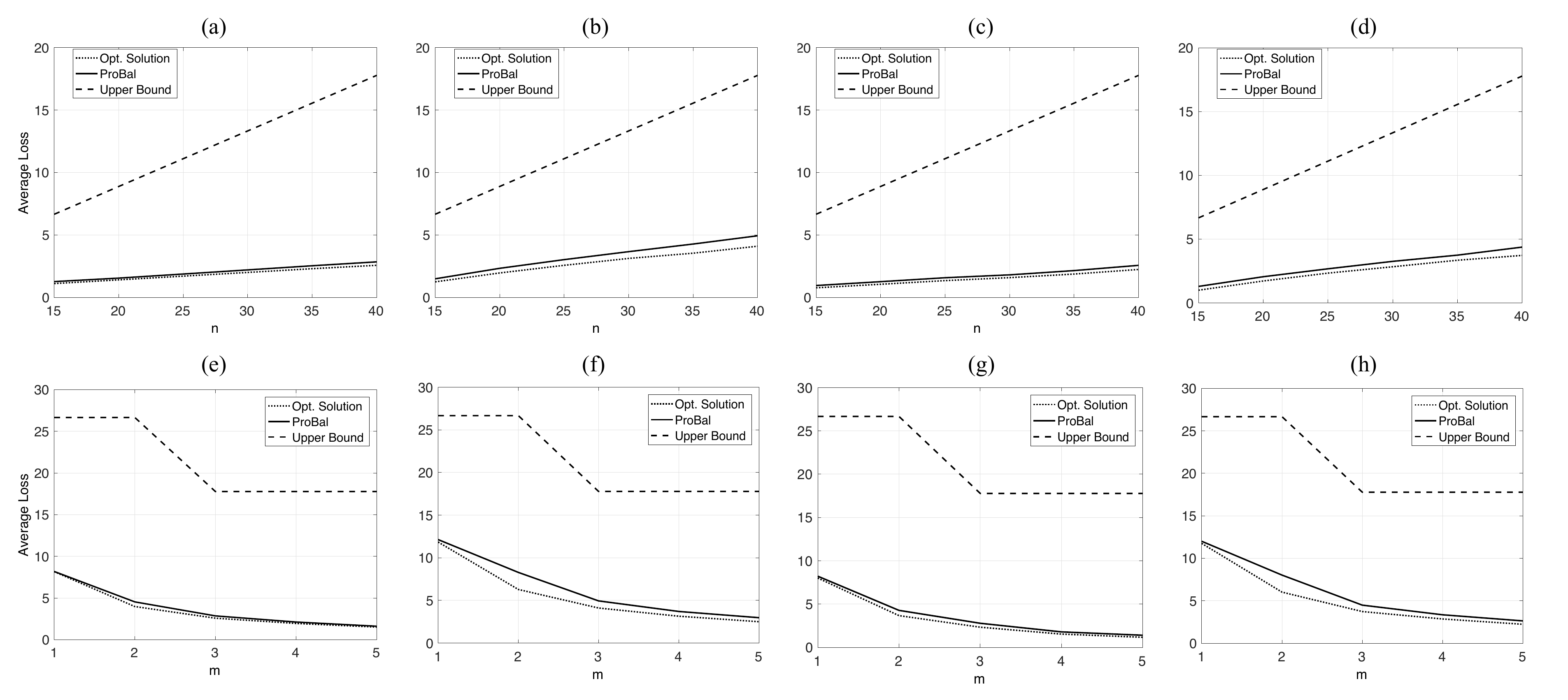}}
\caption{The average loss of ProBal and the optimal solution with respect to the order of the tree (first row), and with respect to the number of the interventions (second row). In the first two columns, $P^{\theta}$ is uniform, while the degree based distribution is used in the simulations for the second two columns. The tree in parts (a), (c), (e) and (g) are created based on Barab\'asi-Albert model model and the bounded degree model is used in the rest.}
\label{fig:synthesized}
\end{center}
\vskip -0.2in
\end{figure*}
%\vspace{-5.5mm}
\section{Experimental Results}
\label{sec:expres}

In this section we evaluate the performance of our proposed algorithm on both synthetic and real data.
%\vspace{-3mm}
\subsection{Synthetic Data}

We generated 1000 instances of trees based on Barab\'asi-Albert (BA) model \cite{barabasi1999emergence, barabasi2016network}, and bounded degree (BD) model created according to Galton-Watson branching process \cite{barabasi2016network}. For both models we considered uniform and degree based distributions for the location of the root of the tree. In the degree based distribution, the probability of vertex $v$ being the root is proportional to its degree.\\
Figure \ref{fig:synthesized} depicts the loss of ProBal as well as the loss for the optimal solution with respect to the order of the tree and the number of the interventions. As shown in this figure, in all cases, the performance of ProBal algorithm is very close to the optimal solution.
There are worst case scenarios where special graphs with specifically designed distributions can reach the upper-bounds, but as seen in Figure \ref{fig:synthesized}, for many distributions, ProBal algorithm works much better than what is predicted by our theoretic upper bound.
%\vspace{-3mm}
\subsection{Real Data}
We examined the performance of ProBal on real data. The graph that we work on is the GRN of E-coli bacteria which we sourced from the RegulonDB database \cite{salgado2006regulondb}.
In GRN, the transcription factors are the main players to activate genes. The interactions between transcription factors and regulated genes in a species genome can be presented by a directed graph. In this graph, links are drawn whenever a transcription factor regulates a gene's expression. Moreover, some of vertices have both functions, i.e., are both transcription factor and regulated gene.
%Figure \ref{fig:EcoliGRN}, shows the GRN of E-coli bacteria, which as mentioned in Section \ref{sec:intro}, has a tree-like structure.
%Figure \ref{fig:componentsGR} depicts the chordal moral components in the GRN that could be obtained after performing the observational test. As seen in this figure, some of the components have a star structure, for which due to the morality property, the location of the root variable is known. Therefore, we do not need to run the ProBal algorithm on star components.
Figure \ref{fig:database} depicts the normalized average loss of ProBal with respect to the total budget for the number of interventions. As seen in this figure, seven interventions are enough to reconstruct more than 95 percent of the network.

%\begin{figure}[t]
%\vskip 0.2in
%\begin{center}
%\centerline{\includegraphics[scale=0.35]{componentsGRN.pdf}}
%\caption{The chordal moral components in the GRN for E-coli.}
%\label{fig:componentsGR}
%\end{center}
%\vskip -0.2in
%\end{figure}

\begin{figure}[t]
\vskip 0.1in
\begin{center}
\centerline{\includegraphics[scale=0.26]{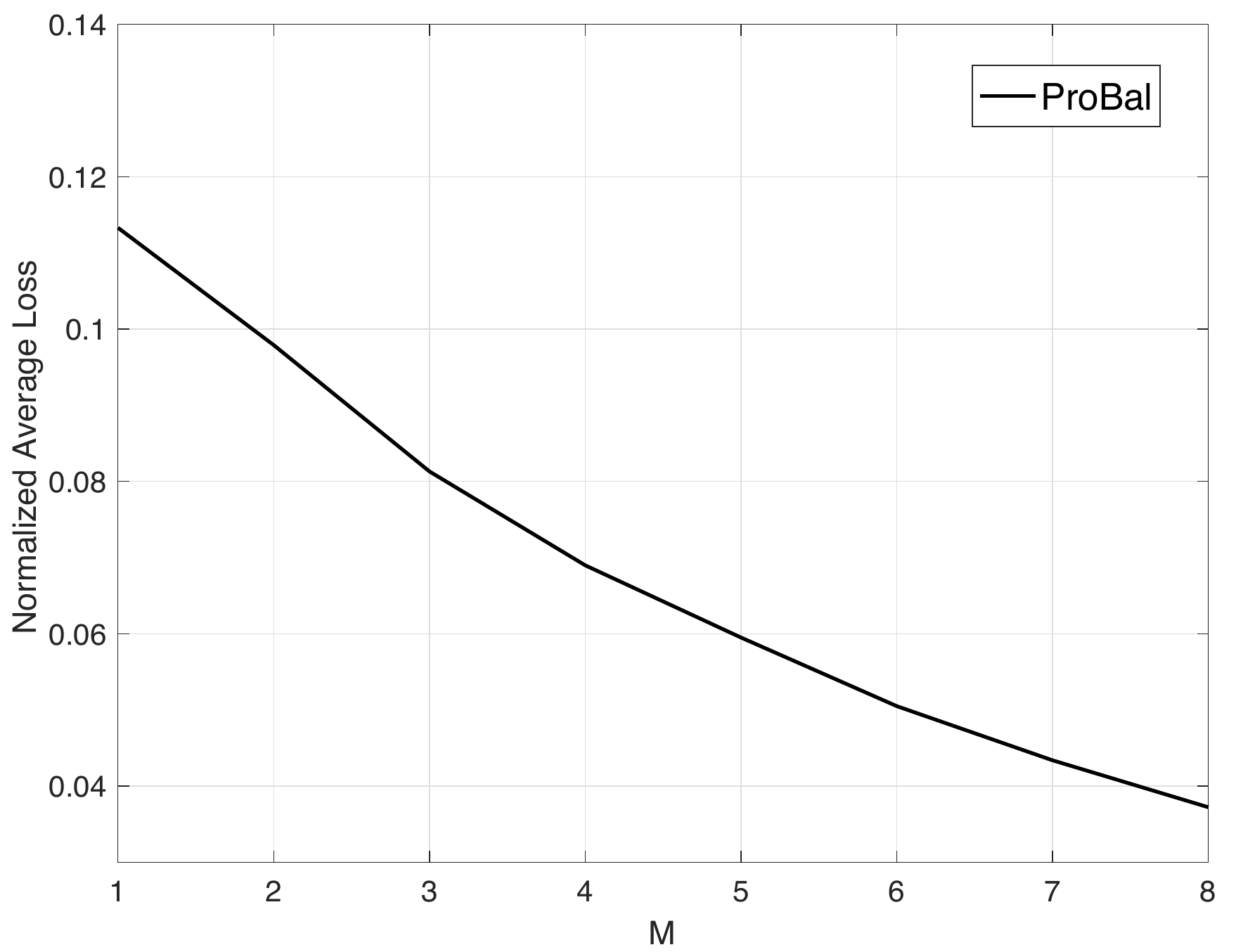}}
\caption{The normalized average loss of ProBal with respect to the total budget for the number of interventions.}
\label{fig:database}
\end{center}
\vskip -0.2in
\end{figure}

\vspace{-3mm}
\section{Conclusion}
\label{sec:conc}

%We studied the problem of experiment design for causal inference for the case of limited number of experiments.
%In our model, each experiment consists of intervening on a single vertex, which makes the model suitable for applications in which intervening on several variables simultaneously is not feasible.
%Also, in our model, experiments are designed merely based on the result of an initial purely observational test, which enables the experimenter to perform the interventional tests in parallel. We assumed that the structure on the variables contains negligible number of triangles compared to the size of the graph. This assumption is satisfied in many applications such as the structure of GRN for some bacteria.
%The main question addressed in this work is that in the described model, for a limited number of experiments, what is the maximum portions of the causal structure that can be learned.
%We first found the optimal solution for this problem and then proposed the ProBal algorithm, which designs the experiments in a time efficient manner.
%We showed that for bounded degree graphs, in the minimax case and in the Bayesian case with uniform prior, our proposed algorithm is a $\rho$-approximation algorithm, where $\rho$ is independent of the order of the underlying graph.
%Finally we examined our proposed algorithm on synthesized as well as real data. The results showed that the performance of the ProBal algorithm is very close to the optimal solution. One direction of future work would be to study this experiment design problem for more general causal structures.

We studied the problem of experiment design for causal inference when only a limited number of experiments are available. In our model, each experiment consists of intervening on a single vertex, which makes the model suitable for applications in which intervening on several variables simultaneously is not feasible. Also, in our model, experiments are designed merely based on the result of an initial purely observational test, which enables the experimenter to perform the interventional tests in parallel. We assumed that the underlying structure on the variables contains negligible number of triangles compared to the size of the graph. This assumption is satisfied in many applications such as the structure of GRN for some bacteria. 
We addressed the following question: ``How much of the causal structure can be learned when only a limited number of experiments are available?'' We characterized the optimal solution to this problem and then proposed an algorithm, which designs the experiments in a time efficient manner. We showed that for bounded degree graphs, in both the minimax setting and the Bayesian settings with uniform prior, our proposed algorithm is a  $\rho$-approximation algorithm, where $\rho$ is independent of the order of the underlying graph. 
We examined our proposed algorithm on synthetic as well as real datasets. The results show that the performance of our proposed algorithm is very close to the optimal theoretical solution. One direction of future work is to extend this experiment design problem to even more general causal structures.

%\pagebreak

\bibliography{ProBal_ref}
\bibliographystyle{icml2016}

\end{document}